\documentclass{article}

\usepackage{microtype}
\usepackage{graphicx}
\usepackage{subfigure}
\usepackage{booktabs} 

\usepackage{amsmath}
\usepackage{amssymb}
\usepackage{amsthm}
\usepackage{comment}
\usepackage{color}

\theoremstyle{plain}
\newtheorem{theorem}{Theorem}

\theoremstyle{definition}

\newtheorem{proposition}[theorem]{Proposition}
\usepackage{algorithm, algorithmic}

\DeclareMathOperator*{\softmax}{softmax}
\DeclareMathOperator*{\argmax}{arg\,max}

 \normalsize

\usepackage{hyperref}

\usepackage[accepted]{icml2019}

\icmltitlerunning{Bayesian Generative Active Deep Learning}

\begin{document}

\twocolumn[
\icmltitle{Bayesian Generative Active Deep Learning}



\icmlsetsymbol{equal}{*}

\begin{icmlauthorlist}
\icmlauthor{Toan Tran}{adl}
\icmlauthor{Thanh-Toan Do}{liv}
\icmlauthor{Ian Reid}{adl}
\icmlauthor{Gustavo Carneiro}{adl}
\end{icmlauthorlist}

\icmlaffiliation{adl}{University of Adelaide, Australia}
\icmlaffiliation{liv}{University of Liverpool}

\icmlcorrespondingauthor{Toan Tran}{toan.m.tran@adelaide.edu.au}

\icmlkeywords{Bayesian inference, Bayesian Deep active learning, Data augmentation, Generative models}

\vskip 0.3in
]



\printAffiliationsAndNotice{}  


\begin{abstract}
	Deep learning models have demonstrated outstanding performance in several problems, but their training process tends to require immense amounts of computational and human resources for training and labeling, constraining the types of problems that can be tackled.
	Therefore, the design of effective training methods that require small labeled training sets is an important research direction that will allow a more effective use of resources.
	Among current approaches designed to address this issue, two are particularly interesting: data augmentation and active learning.  
	Data augmentation achieves this goal by artificially generating new training points, while active learning relies on the selection of the ``most informative'' subset of unlabeled training samples to be labelled by an oracle.
	Although successful in practice, data augmentation can waste computational resources because it indiscriminately generates samples that are not guaranteed to be informative, and active learning selects a small subset of informative samples (from a large un-annotated set) that may be insufficient for the training process. 
	In this paper, we propose a Bayesian generative active deep learning approach that combines active learning with data augmentation -- we provide theoretical and empirical evidence (MNIST, CIFAR-$\{10,100\}$, and SVHN) that our approach has more efficient training and better classification results than data augmentation and active learning. 
\end{abstract}

\section{Introduction}\label{sec:intro}

Deep learning is undoubtedly the dominant machine learning methodology~\cite{esteva2017dermatologist,huang2017densely,kumar2016ask,rajkomar2018scalable}.  Part of the reason behind this success lies in its training process that can be performed with immense and carefully labeled data sets, where the larger the data set, the more effective the training process~\cite{sun2017revisiting}.  However, the labeling of such large data sets demands significant human effort, and the large-scale training process requires considerable computational resources~\cite{sun2017revisiting}.  These training issues have prevented researchers and practitioners from solving a wider range of classification problems, where large labeled data sets are hard to acquire or vast computational resources are unavailable~\cite{litjens2017survey}. 
Addressing these issues is one of the most important problems to be solved in machine learning~\cite{bengio2012deep,gal2017deep,kingma2014semi,settles2012active,krizhevsky2012imagenet,tran2017bayesian,zhu2017generative}.

One of the most successful approaches to mitigate the issue described above relies on the use of a small labeled data set and a large unlabeled data set, where small subsets from the unlabeled set are automatically selected using an acquisition function that assesses how informative those subsets are for the training process.  These selected unlabeled subsets are then labeled by an oracle (i.e., a human annotator), integrated into the labeled data set, which is then used to re-train the model in an iterative training process.  This algorithm is known as (pool-based) active learning~\cite{settles2012active}, and its aim is to reduce the need for large labeled data sets and the computational requirements for training models because it tends to rely on smaller training sets.  Although effective in general, active learning may overfit the informative training sets due to their small sizes.

Alternatively, if the unlabeled data set does not exist, then a possible idea is to use the samples from the labeled set to guide the generation of new artificial training points by sampling from a generative distribution that is assumed to have a particular shape (e.g., Gaussian noise around rigid deformation parameters from the labels)~\cite{krizhevsky2012imagenet} or that have been estimated from a generative adversarial training~\cite{tran2017bayesian}.  This training process is known as data augmentation, which targets the reduction of the need for large labeled training sets.  However, given that the generation of new samples is done without regarding how informative the new sample is for the training process, it is likely that a large proportion of the generated samples will not be important for the training process.  Consequently, data augmentation wastes computational resources, forcing the training process not only to take longer than necessary, but also to be relatively ineffective, particularly at the later stages of the training process, when most of the generated points are likely to be uninformative.

\begin{figure*}[ht]
	\centering
	\includegraphics[width=0.97\textwidth]{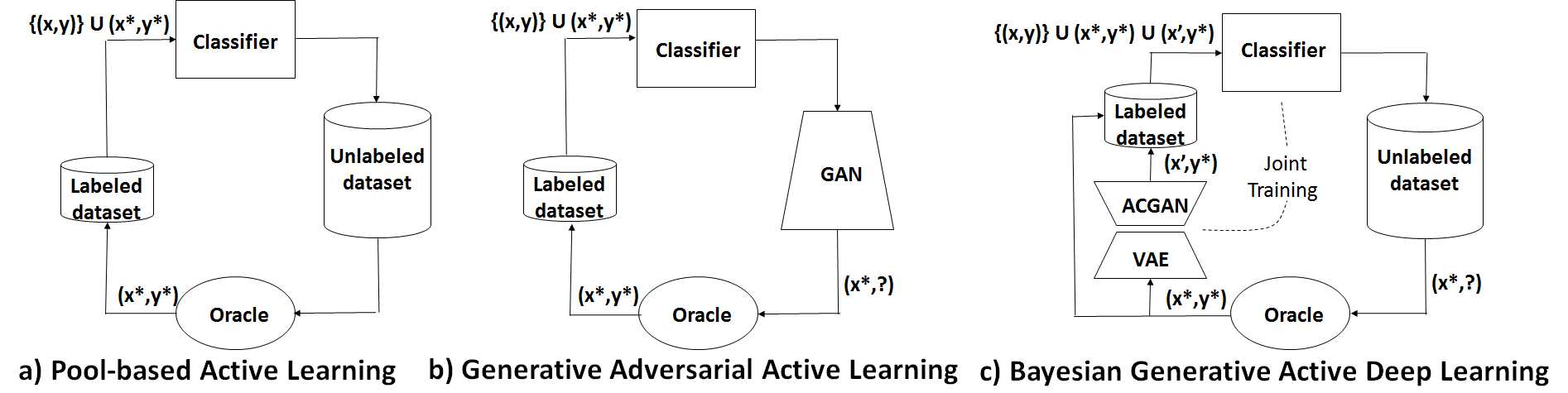}
	\caption{Comparison between (pool-based) active learning~\cite{settles2012active} (a), generative adversarial active learning~\cite{zhu2017generative} (b), and our proposed Bayesian generative active deep learning (c).  The labeled data set is represented by 
		$\{(\mathbf{x},\mathbf{y})\}$, the unlabeled point to be labeled by the oracle is denoted by $\mathbf{x}^*$ (oracle's label is $\mathbf{y}^{*}$), and the point generated by the VAE-ACGAN model is denoted by $\mathbf{x}'$.}
	\label{fig:motivation}
	\vskip -0.2in
\end{figure*}

In this paper, we propose a new Bayesian generative active deep learning method that targets the augmentation of the labeled data set with generated samples that are informative for the training process -- see Fig.~\ref{fig:motivation}.  
Our paper is motivated by the following works: query by synthesis active learning~\cite{zhu2017generative}, Bayesian data augmentation~\cite{tran2017bayesian}, auxiliary-classifier generative adversarial networks (ACGAN)~\cite{odena2017conditional} and variational autoencoder (VAE)~\cite{journals/corr/KingmaW13}.  We assume the existence of a small labeled and a large unlabeled data set, where we use the concept of Bayesian active learning by disagreement (BALD)~\cite{gal2017deep,houlsby2011bayesian} to select informative samples from the unlabeled data set.  These samples are then labeled by an oracle and processed by the VAE-ACGAN to produce new artificial samples that are as informative as the selected samples.  This set of new samples are then incorporated into the labeled data set to be used in the next training iteration.

Compared to a recently proposed generative adversarial active learning~\cite{zhu2017generative}, which relies on an optimization problem to generate new samples (this optimization balances sample informativeness with image generation quality), our approach has the advantage of using acquisition functions that have proved to be more effective~\cite{gal2017deep} than the simple information loss  in~\cite{zhu2017generative}.  Different from our approach that trains the generative and classification models jointly, the approach in~\cite{zhu2017generative} relies on a 2-stage training, where the generator training is independent of the classifier training.  A potential disadvantage of our method is the fact that the whole unlabeled data set needs to be processed by the acquisition function at each iteration, but that is mitigated by the fact that we can sample a much smaller (fixed-size) subset of the unlabeled data set to guarantee the informativeness of the selected samples~\cite{he2013imbalanced}.
An important question about the VAE-ACGAN generation process is how informative the generated artificial sample is, when compared with the active learning selected sample from the unlabeled training set.  We show that this generated sample is theoretically guaranteed to be informative, given a couple of assumptions that are empirically verified.  We run experiments which show that our proposed Bayesian generative active deep learning is advantageous in terms of training efficiency and classification performance, compared with data augmentation and active learning on MNIST, CIFAR-$\{10,100\}$ and SVHN.

\section{Related Work}\label{sec:related_work}

\subsection{Bayesian Active Learning}

In a (pool-based) active learning framework, the learner is initially modeled with a small labeled training set, and it will iteratively search for the ``most informative'' samples from a large unlabeled data set to be labeled by an oracle -- these newly labeled samples are then used to re-model the learner. The information value of an unlabeled sample is estimated by an acquisition function, which is maximized in order to select the most informative samples.  For example, the most informative samples can be selected by maximizing the ``expected informativeness''~\cite{mackay1992information}, or minimizing the ``expected error'' of the learner~\cite{cohn1996active} -- such acquisition functions are hard to optimize in deep learning because they require the estimation of the inverse of the Hessian computed from the expected error with respect to high-dimensional parameter vectors.

Recently, Houlsby et al.~\yrcite{houlsby2011bayesian} proposed the Bayesian active learning by disagreement (BALD) scheme in which the acquisition function is measured by the ``mutual information'' of the training sample with respect to the model parameters. Gal et al.~\yrcite{gal2017deep} pointed out that, in deep active learning, the evaluation of this function is based on model uncertainty, which in turn requires the approximation of the posterior distribution of the model parameters. These authors also introduced the use of Monte Carlo (MC) dropout method~\cite{gal2016dropout} to approximate this and other commonly used acquisition functions. This approach~\cite{gal2017deep} is shown to work well in practice despite the poor convergence of the MC approximation. In our proposed approach, we also use this method to approximate the BALD acquisition function in the active selection process. 

\subsection{Data Augmentation}
\label{sec:lit_rev_DA}

In active learning, it is assumed that a model can be trained with a small data set. That assumption is challenging in the estimation of a deep learning model since it often requires large labeled data sets to avoid over-fitting. One reasonable way to increase the labeled training set is with data augmentation that artificially generates new synthetic training samples~\cite{krizhevsky2012imagenet}.  Gal et al.~\yrcite{gal2017deep} also emphasized the importance of data augmentation for the development of deep active learning. Data augmentation can be performed with ``label-preserving'' transformations~\cite{krizhevsky2012imagenet,Simard2003,yaeger1996} -- this is known as ``poor's man'' data augmentation (PMDA)~\cite{tanner1991tools,tran2017bayesian}.  Alternatively, Bayesian data augmentation (BDA) trains a deep generative model (using the training set), which is then used to produce new artificial training samples~\cite{tran2017bayesian}. Compared to PMDA, BDA has been shown to have a better theoretical foundation and to be more beneficial in practice~\cite{tran2017bayesian}. One of the drawbacks of data augmentation is that the generation of new training points is driven by the likelihood that the generated samples belong to the training set -- this implies that the model produces samples that are likely to be close to the generative distribution mode.  Unfortunately, as the training process progresses, these points are the ones more likely to be correctly classified by classifier, and as a result they are not informative.
The combination of active learning and data augmentation proposed in this paper addresses the issue above, where the goal is to continuously generate informative training samples that not only are likely to belong to the learned generative distribution, but are also informative for the training process -- see Fig.~\ref{fig:lit_review}.

\begin{figure*}[ht]
	\centering
	\includegraphics[width=0.85\textwidth]{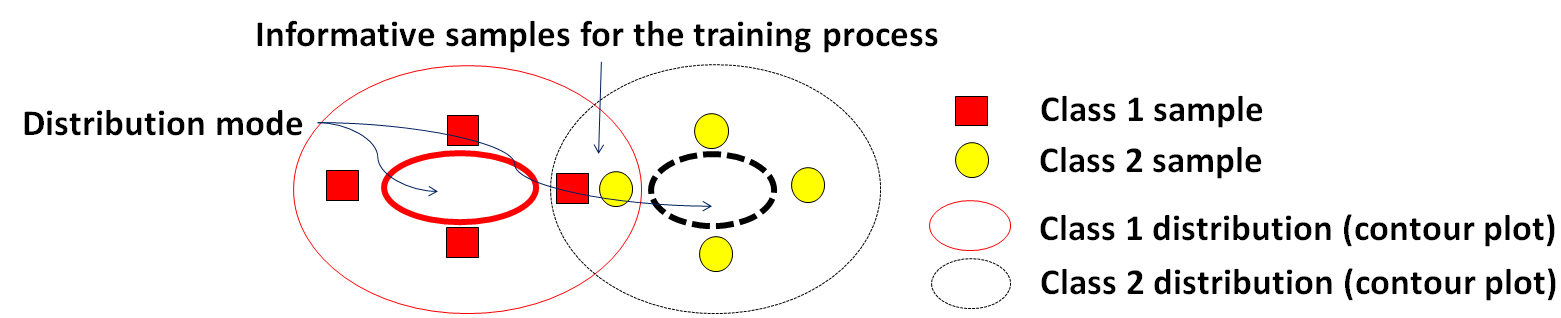}
	\caption{We target the generation of samples that belong to the generative distribution learned from the training set, and that are also informative for the training process. In particular, we aim to generate synthetic samples belonging to the intersection of different class distributions known as ``disagreement region''~\cite{settles2012active}. These generated instances are informative for the training process since the learning model is uncertain about them~\cite{houlsby2011bayesian}.}
	\label{fig:lit_review}
	\vskip -0.2in
\end{figure*}

\subsection{Generative Active Learning}

The training process in active learning can be significantly accelerated by actively generating informative samples. Instead of querying most informative instances from an unlabeled pool, Zhu \& Bento~\yrcite{zhu2017generative} introduced a generative adversarial active learning (GAAL) model to produce new synthetic samples that are informative for the current model. The major advantage of their algorithm is that it can generate rich representative training data with the assumptions that the GAN model has been pre-trained and the optimization during generation is solved efficiently. Nevertheless, this approach has a couple of limitations that make it challenging to be applied in deep learning. First, the 
acquisition function must be simple to compute and optimize (e.g., distance to classification hyper-plane) because it will be used by the generative model during the  sample generation process -- such simple acquisition functions have been shown to be not quite useful in active learning~\cite{gal2017deep}. Also, the GAN model in~\cite{zhu2017generative} is not fine-tuned as  training progresses since it is pre-trained only once before generating new instances --  as a result, the generative and discriminative models do not ``co-evolve''. 

In contrast, following the standard active learning, our Bayesian Generative Active Deep Learning first queries the unlabeled data set samples based on their ``information content'', and conditions the generation of a new synthetic sample on this selected sample. Moreover, the learner and the generator are jointly trained in our approach, allowing them to ``co-evolve'' during training. We show empirically that, in our proposed approach, a synthetic sample generated from the most informative sample belongs to its sufficiently small neighborhood. More importantly, the value of the acquisition function at the generated sample is mathematically shown to be closed to its optimal value (at the most informative sample), and the synthetic instance, therefore, can also be considered to be informative.

\subsection{Variational Autoencoder Generative Adversarial Networks}

Generative Adversarial Network (GAN)~\cite{goodfellow2014generative} is one of the most studied deep learning models. GANs typically contain two components: a generator that learns to map a latent variable to a sample data, and a discriminator that aims to guide the generator to produce realistically looking samples. The generative performance of GAN is often evaluated by both the quality and the diversity of the synthetic instances. There have been several extensions proposed to improve the quality of the GAN generated images, such as CGAN~\cite{Mirza} and ACGAN~\cite{odena2017conditional}. In order to tackle the low diversity problem (known as ``mode collapse''), Larsen et al.~\yrcite{pmlr-v48-larsen16} introduced a variational autoencoder generative adversarial network (VAE-GAN) that combines a variational autoencoder (VAE)~\cite{journals/corr/KingmaW13} and a GAN in which these networks are connected by a generator/decoder~\cite{zhang2017gans}. We utilize both ACGAN and VAE-GAN  in our proposed Bayesian Generative Active Deep Learning framework, but with the aim of improving the classification performance.  

\section{``Information-Preserving'' Data Augmentation for Active Learning}\label{sec:da_framework}

\subsection{Bayesian Active Learning by Disagreement (BALD)}
\label{sec:bald}

Let us denote the initial labeled data by $\mathcal D=\{(\mathbf{x}_i, \mathbf{y}_i)\}_{i=1}^N$, where $\mathbf{x}_i \in \mathcal X \subseteq \mathbb R^{d}$ is the data sample labeled with $\mathbf{y}_i \in \mathcal C = \{1, 2, \ldots, C\}$ ($C=$ \# classes). By using Bayesian deep learning framework, we can obtain an estimate of the posterior of the parameters $\mathbf\theta$ of the model $\mathcal M$ given $\mathcal D$, namely $p(\mathbf\theta | \mathcal D)$.  In Bayesian Active Learning by Disagreement (BALD) scheme~\cite{houlsby2011bayesian}, the most informative sample $\mathbf{x}^{*}$ is selected from the (unlabeled) pool data $\mathcal D_{\text{pool}}$ by~\cite{houlsby2011bayesian}:  
\begin{align}\label{acq_rule}
\mathbf{x}^{*} & = \argmax_{\mathbf{x} \in \mathcal D_{\text{pool}}}a(\mathbf{x}, \mathcal M) \notag\\
& = \argmax_{\mathbf{x} \in \mathcal D_{\text{pool}}} H[\mathbf{y} | \mathbf{x}, \mathcal D] - \mathbb E_{\mathbf\theta \sim p(\mathbf\theta | \mathcal D)}[H[\mathbf{y} | \mathbf{x}, \mathbf\theta]],
\end{align}
where $a(\mathbf{x}, \mathcal M)$ is the acquisition function, $H[\mathbf{y} | \mathbf{x}, \mathcal D]$ and $H[\mathbf{y} | \mathbf{x}, \mathbf\theta]$ are represented by the Shannon entropy~\cite{shannon2001mathematical} of the prediction $p(\mathbf{y} | \mathbf{x}, \mathcal D)$ and the distribution $p(\mathbf{y} | \mathbf{x}, \mathbf\theta)$, respectively.  The sample $\mathbf{x}^{*}$ is labeled with $\mathbf{y}^{*}$ (by an oracle), and the labeled data set is updated for the next training iteration: $\mathcal D \leftarrow \mathcal D \cup {(\mathbf{x}^{*}, \mathbf{y}^{*})}$.  That active selection framework is repeated until convergence.

In order to estimate the acquisition function in~\eqref{acq_rule}, Gal et al.~\yrcite{gal2017deep} introduced the Monte Carlo (MC) dropout method. This objective function can be approximated by its sample mean~\cite{gal2017deep}:
\begin{multline}\label{acq_approx}
a(\mathbf{x},\mathcal M) \approx 
-\sum_c \left( \frac{1}{T}\sum_t \hat p_c^t\right) \log \left( \frac{1}{T}\sum_t \hat p_c^t\right) \\ 
+\frac{1}{T}\sum_{c, t}\hat p_c^t \log \hat p_c^t, 
\end{multline}
where $T$ is the number of dropout iterations, $\hat {\mathbf{p}}^t =[\hat p^t_1, \ldots, \hat p^t_C] = \softmax (f(\mathbf{x};\mathbf\theta^t))$, with $f$ representing the network function parameterized by $\mathbf\theta^t$ that is sampled from an estimate of the (commonly intractable) posterior $p(\mathbf\theta | \mathcal D)$ at the $t$-th iteration.

\subsection{Generative Model and Bayesian Data Augmentation}
\label{sec:bda}

In the iterative Bayesian data augmentation (BDA) framework~\cite{tran2017bayesian}, each iteration consists of two steps: synthetic data generation and model training.  The BDA model comprises a generator (that generates new training samples from a latent variable), a discriminator (that discriminates between real and fake samples) and a classifier (that classifies the samples into one of the classes in $\mathcal C$). At the first step, given a latent variable $\mathbf{u}$ (e.g., a multivariate Gaussian variable) and a class label $\mathbf{y} \in \mathcal C$, the generator represented by a function $g(.)$ maps the tuple $(\mathbf{u}, \mathbf{y})$ to a data point $\mathbf{x}^a = g(\mathbf{u},\mathbf{y}) \in \mathcal X$, and $(\mathbf{x}^a, \mathbf{y})$ is then added to $\mathcal D$ for model training. In~\cite{tran2017bayesian}, the authors also showed a weak convergence proof that is related to the improvement of the posterior distribution $p(\mathbf\theta |\mathcal D)$. 

\subsection{Bayesian Generative Active Deep Learning}\label{sec:bgadl}

The main technical contribution of this paper consists of combining BALD and BDA for generating new labeled samples that are informative for the training process (see Fig.~\ref{fig:vaeacgan}).

\begin{figure*}[ht]
	\centering
	\includegraphics[width=0.85\textwidth]{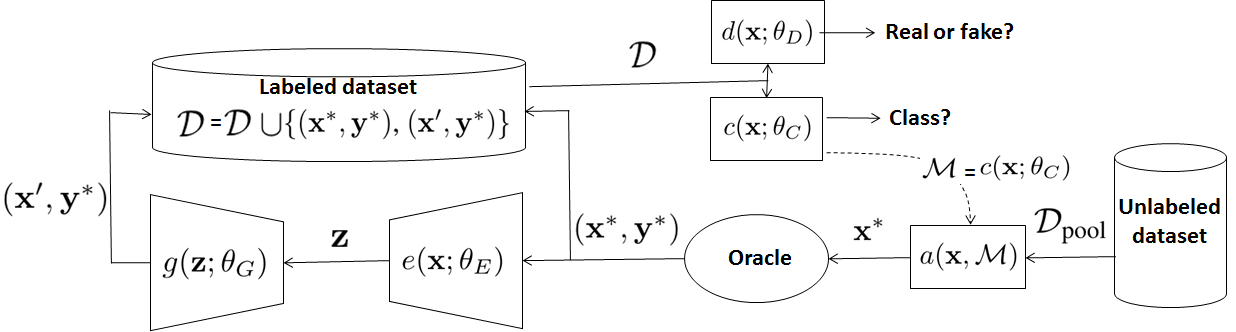}
	\caption{Network architecture of our proposed model.}
	\label{fig:vaeacgan}
	\vskip -0.15in
\end{figure*}

We modify BDA~\cite{tran2017bayesian} by conditioning the generation step on a sample $\mathbf{x}$ and a label $\mathbf{y}$ (instead of the latent variable $\mathbf{u}$ and label $\mathbf{y}$ in BDA). 
More specifically, the most informative sample $\mathbf{x}^{*}$ selected by solving~\eqref{acq_rule} using the estimation~\eqref{acq_approx} is pushed to go through a variational autoencoder (VAE)~\cite{journals/corr/KingmaW13}, which contains an encoder $e(.)$ and a decoder $g(.)$, in order to generate the sample $\mathbf{x}'$, as follows:
\begin{equation}\label{eq:gen_spl}
\mathbf{x}' = g(e(\mathbf{x}^{*})).
\end{equation}

\begin{figure}[ht]
	\vskip 0.2in
	\centering
	\includegraphics[width=0.35\textwidth]{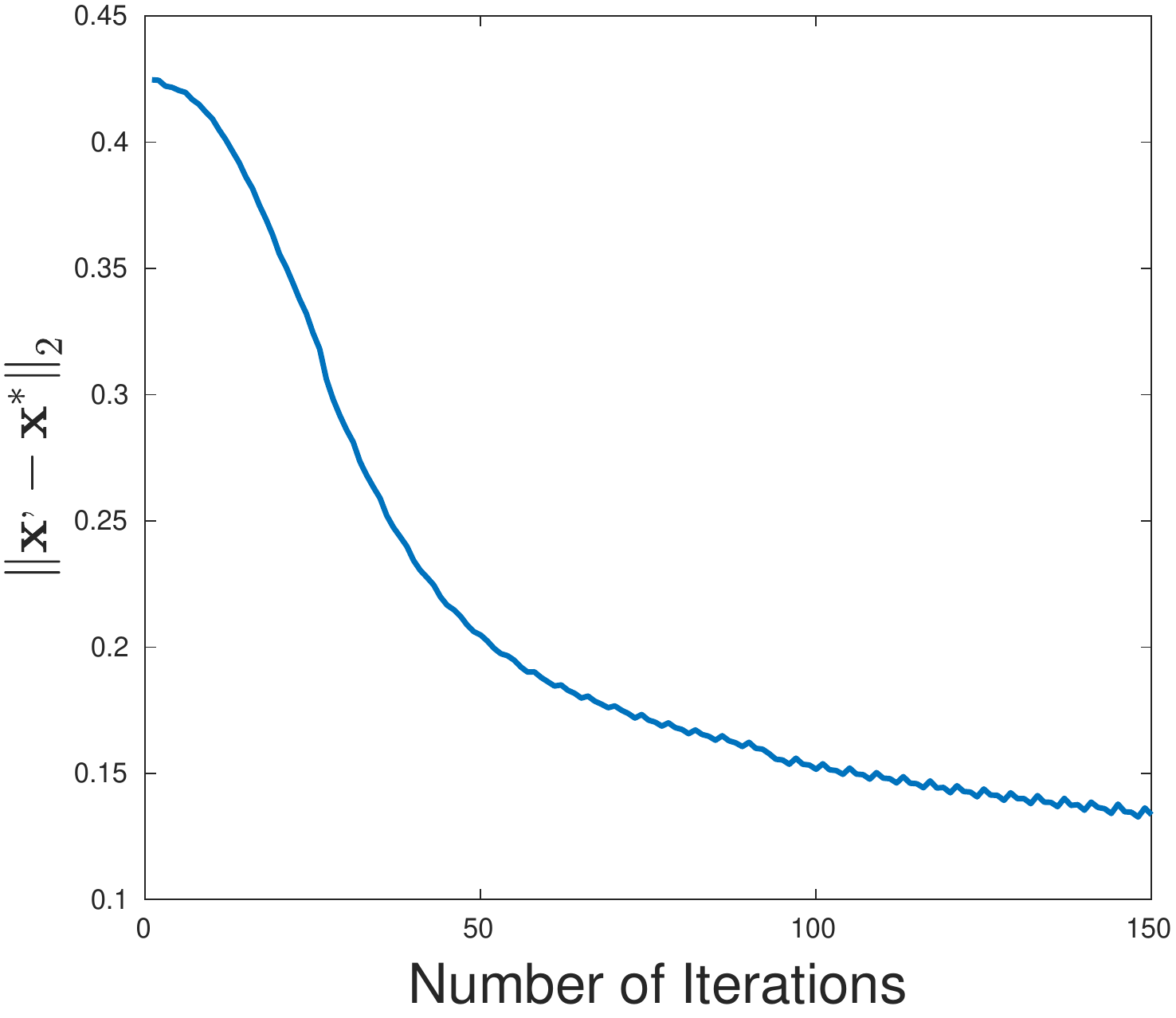}
	\caption{Reduction of $\|\mathbf{x}' - \mathbf{x}^{*}\|$ as the training of the VAE model progresses (on CIFAR-100 using ResNet-18).}
	\label{fig:reconstruction_loss}
	\vskip -0.2in
\end{figure}

The training process of a VAE is performed by minimizing the ``reconstruction loss'' $\mathcal \ell(\mathbf{x}^{*}, g(e(\mathbf{x}^{*}))$~\cite{journals/corr/KingmaW13}, where if the number of training iterations is sufficiently large, we have:
\begin{equation}\label{eq:point_appx}
\|\mathbf{x}' - \mathbf{x}^{*}\| < \varepsilon,
\end{equation}
with $\varepsilon$ representing an arbitrarily small positive constant -- see Fig.~\ref{fig:reconstruction_loss} for an evidence for that claim.  

The label of $\mathbf{x}'$ is assumed to be $\mathbf{y}^{*}$ (i.e., the oracle's label for $\mathbf{x}^{*}$) and the current labeled data set is then augmented with $(\mathbf{x}^{*}, \mathbf{y}^{*})$ and $(\mathbf{x}', \mathbf{y}^{*})$, which are used for the next training iteration. 
To evaluate the ``information content'' of the generated sample $\mathbf{x}'$, which is measured by the value of the acquisition function at that point, $a(\mathbf{x}', \mathcal M)$, we consider the following proposition.

\begin{proposition}\label{propos_inf}
	Assuming that there exists the gradient of the acquisition function $a(\mathbf{x}, \mathcal{M})$ with respect to the variable $\mathbf{x}$, namely $\nabla_x a(\mathbf{x}, \mathcal M)$, and that $\mathbf{x}^{*}$ is an interior point of $\mathcal D_{\text{pool}}$, then $a(\mathbf{x'}, \mathcal M) \approx a(\mathbf{x}^{*}, \mathcal M)$ (i.e., the absolute difference between these values are within a certain range). Consequently, the sample $\mathbf{x}'$ generated from the most informative sample $\mathbf{x}^{*}$ by ~\eqref{eq:gen_spl} is also informative.
\end{proposition}
\begin{proof}
	Given the assumptions of Proposition~\ref{propos_inf}, and due to the fact that $\mathbf{x}^{*}$ is a local maximum of function $a(\mathbf{x}, \mathcal{M})$~\eqref{acq_rule}, then $\mathbf{x}^{*}$ is a critical point of $a(\mathbf{x}, \mathcal M)$, i.e.,
	\begin{equation}\label{eq:crit_point}
	\nabla_x a(\mathbf{x}^{*}, \mathcal M) = \mathbf{0}.
	\end{equation}
	Condition~\eqref{eq:point_appx}, which is empirically verified by Fig.~\ref{fig:reconstruction_loss}, indicates that $\mathbf{x'}$ belongs to a sufficiently small neighborhood of $\mathbf{x}^{*}$. Therefore, by  using the first order Taylor approximation of the function $a(\mathbf{x'}, \mathcal M)$ at the point $\mathbf{x}^{*}$ and ~\eqref{eq:crit_point}, we obtain
	\begin{align}\label{taylor_app}
	a(\mathbf{x}', \mathcal M) & \approx a(\mathbf{x}^{*}, \mathcal M) + \nabla_x a(\mathbf{x}^{*}, \mathcal M)^{T}(\mathbf{x}' - \mathbf{x}^{*}) \notag\\
	& \approx a(\mathbf{x}^{*}, \mathcal M), 
	\end{align}  
	where $T$ denotes the transpose operator. 
	Thus, the synthetic sample $\mathbf{x}'$ can also be considered informative.
\end{proof}

\section{Implementation}\label{sec:implementation}

Our network, depicted in Fig.~\ref{fig:vaeacgan}, comprises four components: a classifier $c(\mathbf{x};\theta_C)$, an encoder $e(\mathbf{x};\theta_E)$, a decoder/generator $g(\mathbf{z};\theta_G)$ and a discriminator $d(\mathbf{x};\theta_D)$. The classifier $c(.)$ can be represented  by any modern deep convolutional neural network classifier~\cite{lecun1998gradient,ResNet,he2016identity}, making our model quite flexible in the sense that we can use the top-performing classifier available in the field.
Also, the generative part of the model is based on ACGAN~\cite{odena2017conditional} and VAE-GAN~\cite{pmlr-v48-larsen16}, where the VAE decoder is also the generator of the GAN model --  our model is referred to as VAE-ACGAN.

The VAE-GAN loss function~\cite{pmlr-v48-larsen16,zhang2017gans} was formed by adding the reconstruction error in VAE to the GAN loss in order to penalize both \emph{unrealisticness} and \emph{mode collapse} in GAN training. Following that, the VAE-ACGAN loss of our proposed model is defined by
\begin{align}\label{vae_acgan_loss}
\mathcal L = \mathcal{L}_{\text{VAE}} + \mathcal{L}_{\text{ACGAN}},
\end{align}
with the VAE loss~\cite{journals/corr/KingmaW13,pmlr-v48-larsen16} represented by a combination of the reconstruction loss $\mathcal{L}_{\text{rec}}$ and the regularization prior $\mathcal{L}_{\text{prior}}$, i.e., 
\begin{align}\label{vae_loss}
& \mathcal{L}_{\text{VAE}} = \mathcal{L}_{\text{rec}}+ \mathcal{L}_{\text{prior}} \notag \\
&  = \mathcal{L}_{\text{rec}}(\mathbf{x}, g(e(\mathbf{x}; \mathbf{\theta}_E); \mathbf{\theta}_G) + D_{\text{KL}}(q(\mathbf{z}|\mathbf{x})\| p(\mathbf{z})),
\end{align}
where  $\mathbf{z} = e(\mathbf{x};\theta_E)=q(\mathbf{z}|\mathbf{x})$, $p(\mathbf{z})$ is the prior distribution of $\mathbf{z}$ (e.g., $\mathcal N(\mathbf{0}, \mathbf{I})$) and $D_{\text{KL}}(q \| p) = \int q\log \dfrac{p}{q}$ denotes the Kullback-Leibler divergence operator. The ACGAN loss~\cite{odena2017conditional} in (\ref{vae_acgan_loss}) is computed by 
\begin{align}\label{acgan_loss}
& \mathcal{L}_{\text{ACGAN}}  = \log (d(\mathbf{x}; \mathbf{\theta}_D)) + \log (1-d(g(\mathbf{z}; \mathbf{\theta}_G); \mathbf{\theta}_D)) \notag\\
& + \log (1-d(g(\mathbf{u}; \mathbf{\theta}_G); \mathbf{\theta}_D)) + \log (\softmax(c(\mathbf{x}; \mathbf{\theta}_C))) \notag \\
& + \log (\softmax(c(g(\mathbf{z}; \mathbf{\theta}_G); \mathbf{\theta}_C))) \notag\\
& +\log (\softmax(c(g(\mathbf{u}; 
\mathbf{\theta}_G); \mathbf{\theta}_C))),
\end{align}
where $\mathbf{u} \sim \mathcal N(\mathbf{0}, \mathbf{I})$.
The training process of the VAE-ACGAN network is presented in Algorithm~\ref{alg_BGAL}.   

\begin{algorithm}[tb]
	\caption{Bayesian Generative Active Learning}\label{alg_BGAL}
	\begin{algorithmic}
		\STATE Initialize network parameters $\mathbf{\theta}_E, \mathbf{\theta}_G, \mathbf{\theta}_C, \mathbf{\theta}_D$, and pre-train the classifier $c(\mathbf{x};\theta_C)$ with $\mathcal D$
		\REPEAT 
		\STATE Pick the most informative $\mathbf{x^{*}}$ from $\mathcal D_{\text{pool}}$ with $\mathbf{x}^{*}=\arg\max_{\mathbf{x}\in\mathcal D_{\text{pool}}}a(\mathbf{x},\mathcal M)$ in~\eqref{acq_rule} and~\eqref{acq_approx}, where $\mathcal M$ is represented by the classifier $c(\mathbf{x};\theta_C)$; 
		\STATE Request the oracle to label the selected sample, which forms $(\mathbf{x}^{*}, \mathbf{y}^{*})$
		\STATE $\mathbf{z} \gets e(\mathbf{x^{*}};\theta_E)$
		\STATE $\mathcal L_{\text{prior}} \gets D_{\text{KL}}(q(\mathbf{z} | \mathbf{x}^{*}) \| p(\mathbf{z}))$ 
		\STATE $\mathbf{x}' = g(e(\mathbf{x}^{*});\theta_G)$
		\STATE $\mathcal{L}_{\text{rec}} \gets \mathcal{L}_{\text{rec}}(\mathbf{x}^{*}, \mathbf{x}')$
		\STATE Sample $\mathbf{u} \sim \mathcal N(\mathbf{0}, \mathbf{I})$
		\STATE $\mathcal L_{\text{ACGAN}} \gets \log(d(\mathbf{x}^{*})) + \log(1-d(\mathbf{x}')) + \log(1-d(g(\mathbf{u})))+\log(\softmax(c(\mathbf{x}^{*})))+\log(\softmax(c(\mathbf{x}')))+ \log(\softmax(c(g(\mathbf{u}))))$
		\STATE $\mathbf\theta_E \gets \mathbf\theta_E-\nabla_{\mathbf\theta_E}(\mathcal L_{\text{rec}} + \mathcal L_{\text{prior}})$
		\STATE $\mathbf\theta_G \gets \mathbf\theta_G-\nabla_{\mathbf\theta_G}(\gamma\mathcal L_{\text{rec}}-\mathcal L_{\text{ACGAN}})$ (parameter $\gamma=0.75$~\cite{pmlr-v48-larsen16} in our experiments)
		\STATE $\mathbf\theta_D \gets \mathbf\theta_D-\nabla_{\mathbf\theta_D}\mathcal L_{\text{ACGAN}}$
		\STATE $\mathbf\theta_C \gets \mathbf\theta_C-\nabla_{\mathbf\theta_C}\mathcal L_{\text{ACGAN}}$
		\UNTIL{convergence}
	\end{algorithmic}
\end{algorithm}

\section{Experiments and Results}\label{sec:experiment}

In this section, we assess quantitatively our proposed Bayesian Generative Active Deep Learning in terms of classification performance measured by the top-1 accuracy \footnote{code will be available soon}. 
In particular, our proposed algorithm, active learning using ``information-preserving'' data augmentation (AL w. VAEACGAN) is compared with active learning using BDA (AL w. ACGAN), BALD without using data augmentation (AL without DA), BDA without active learning (BDA)~\cite{tran2017bayesian} (using full and partial training sets), and random selection as a function of the number of acquisition iterations and the percentage of training samples.
Our experiments are performed on MNIST~\cite{lecun1998gradient}, CIFAR-10, CIFAR-100~\cite{krizhevsky2012imagenet}, and SVHN~\cite{netzer2011reading}.
MNIST~\cite{lecun1998gradient} contains handwritten digits, (with $60,000$ training and $10,000$ testing samples, and 10 classes), CIFAR-10~\cite{krizhevsky2012imagenet} is composed of $32 \times 32$ color images (with $50,000$ training and $10,000$ testing samples, and $10$ classes), CIFAR-100~\cite{krizhevsky2012imagenet} is similar to CIFAR-10, but with $100$ classes, and SVHN~\cite{netzer2011reading} contains $32 \times 32$ street view house numbers (with $73257$ training samples and $26032$ testing samples, and $10$ classes).

Given that our approach can use any classifier, we test it with ResNet18~\cite{ResNet} and ResNet18pa~\cite{he2016identity}, which have shown to produce competitive classification results in several tasks.
The sample acquisition setup for each data set is: 1) the number of  samples in the initial training set is $1,000$ for MNIST, $5,000$ for CIFAR-10, $15,000$ for CIFAR-100, and $10,000$ for SVHN (the initial data set percentage was empirically set -- with values below these amounts, we could not make the training process converge); 2) the number of acquisition iterations is $150$ ($50$ for SVHN), where at each iteration $100$ ($500$ for SVHN) samples are selected from $2,000$ randomly selected samples of the unlabeled data set $\mathcal D_{\text{pool}}$ (this fixed number of randomly selected samples almost certainly contains the most informative sample~\cite{he2013imbalanced}).
The training process was run with the following hyper-parameters: 1) the classifier $c(\mathbf{x}; \theta_C)$ used stochastic gradient descent with (lr=0.01, momentum=0.9); 2) the encoder $e(\mathbf{x};\theta_E)$, generator $g(\mathbf{z};\theta_G)$ and discriminator $d(\mathbf{x}; \theta_D)$ used Adam optimizer with (lr=0.0002, $\beta_1=0.5$, $\beta_2=0.999$); the mini-batch size is $100$ for all cases.

\begin{figure*}[th]
	\centering
	{\includegraphics[width=0.24\textwidth]{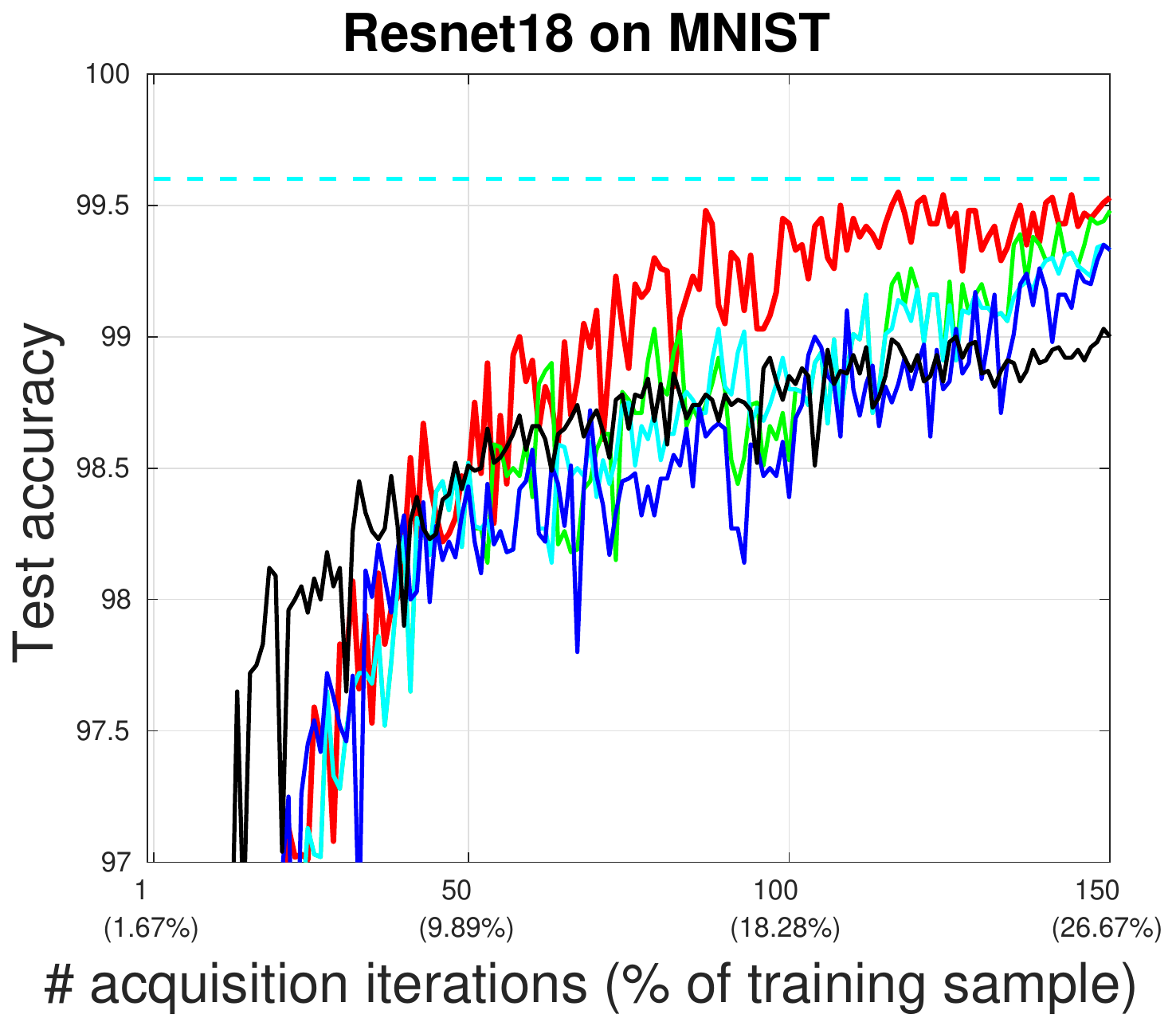}}
	{\includegraphics[width=0.24\textwidth]{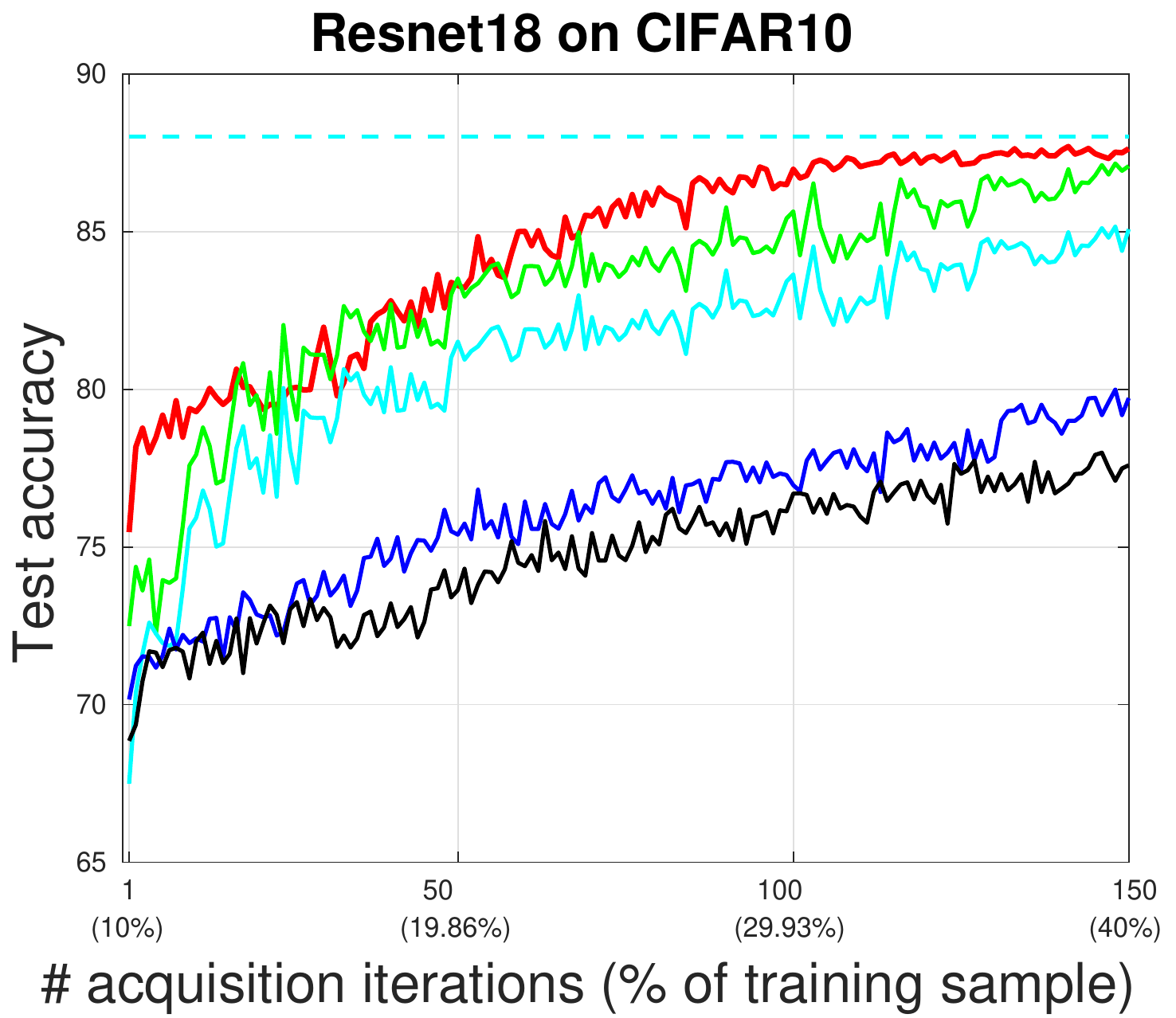}}
	{\includegraphics[width=0.24\textwidth]{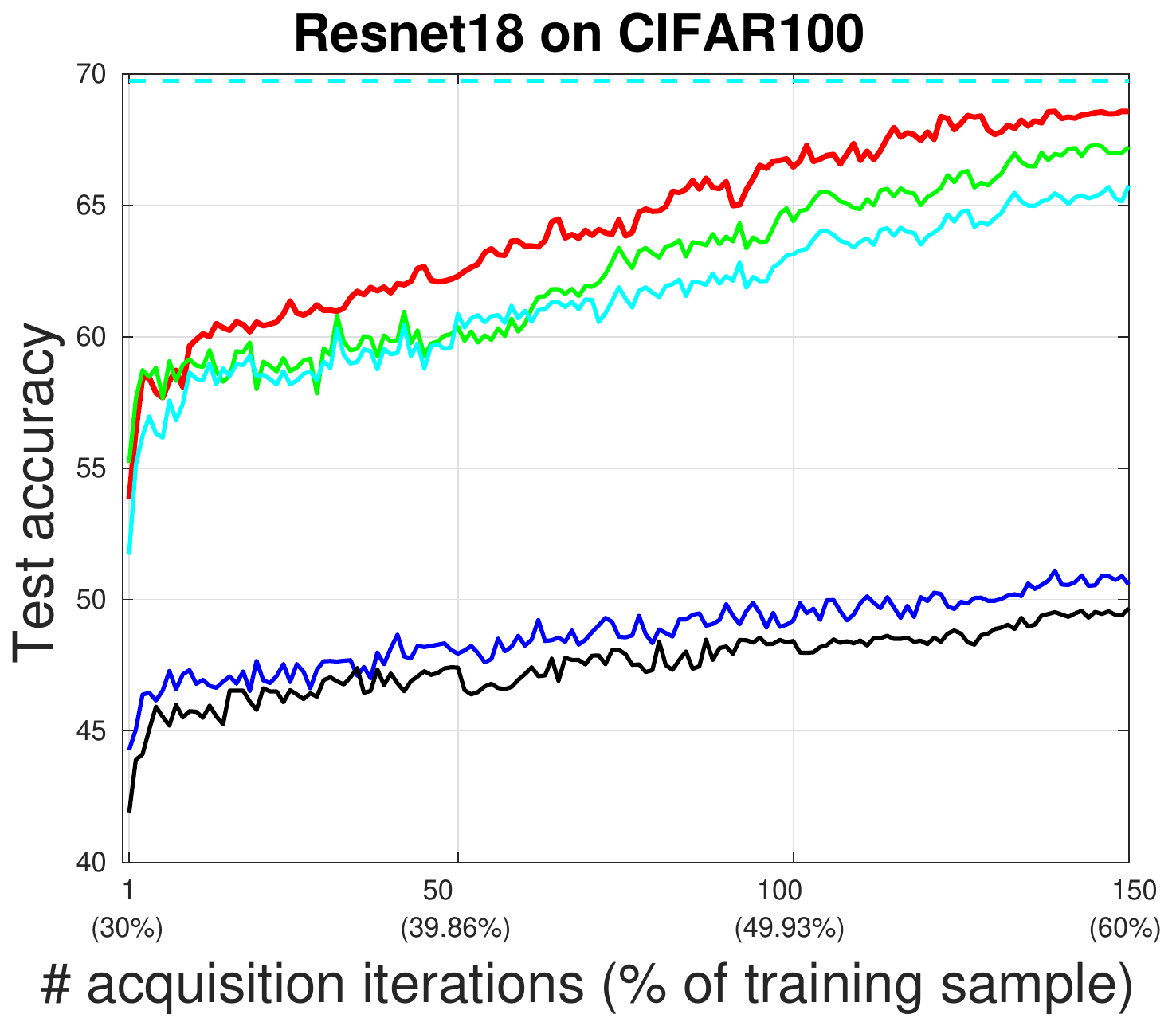}}
	{\includegraphics[width=0.24\textwidth]{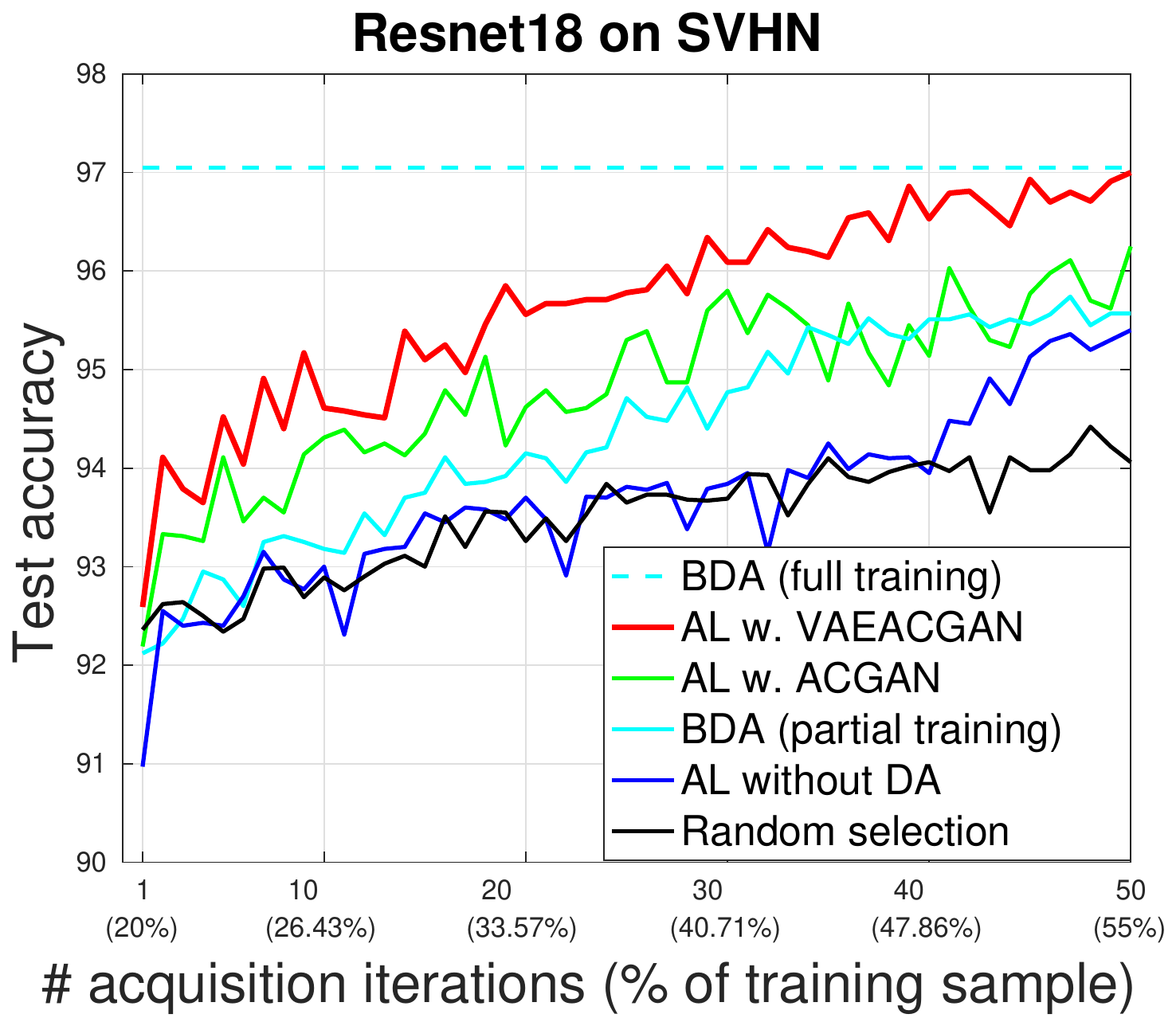}}\\
	\subfigure[MNIST]{\includegraphics[width=0.24\textwidth]{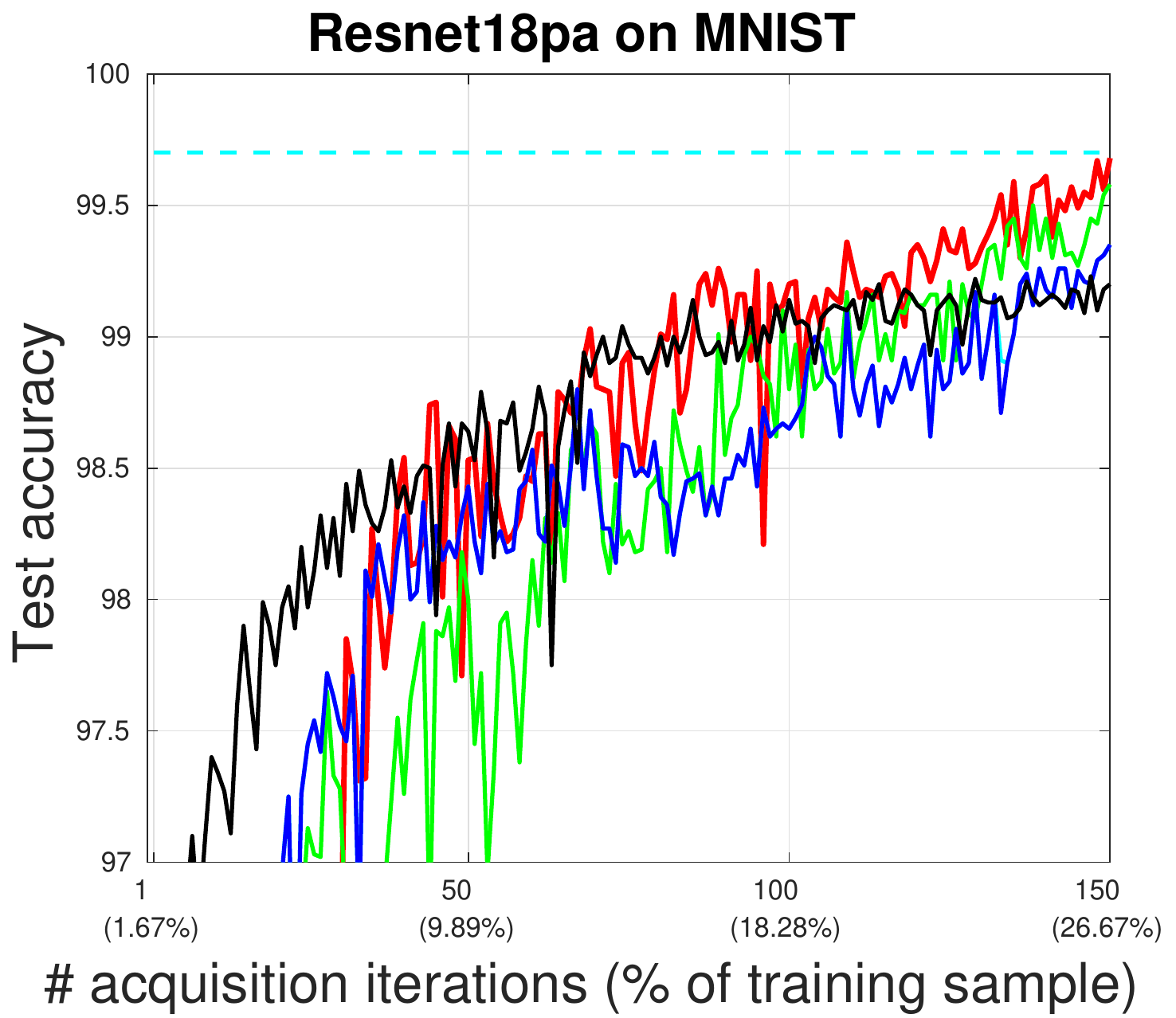}}
	\subfigure[CIFAR-10]{\includegraphics[width=0.24\textwidth]{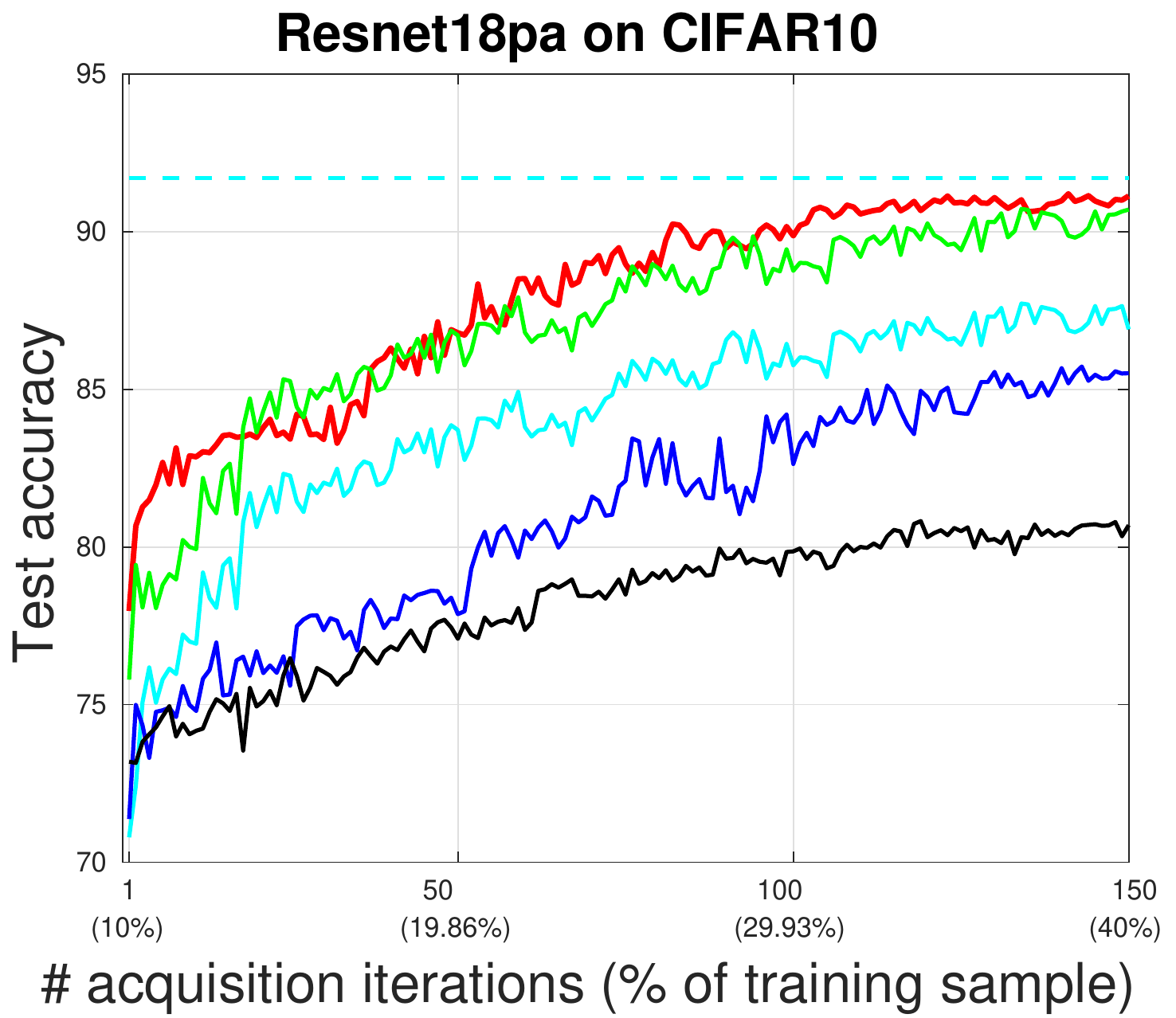}}
	\subfigure[CIFAR-100]{\includegraphics[width=0.24\textwidth]{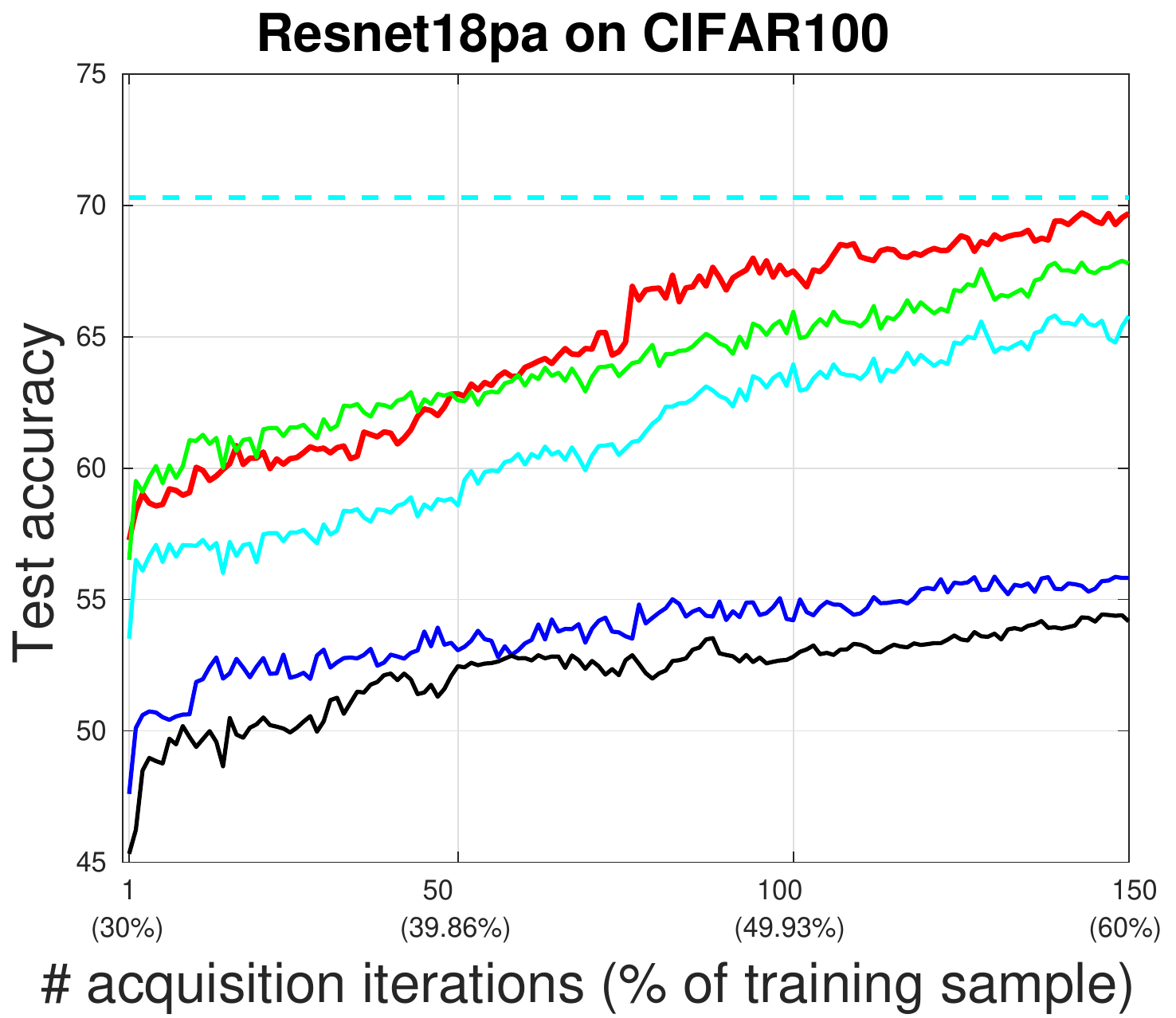}}
	\subfigure[SVHN]{\includegraphics[width=0.24\textwidth]{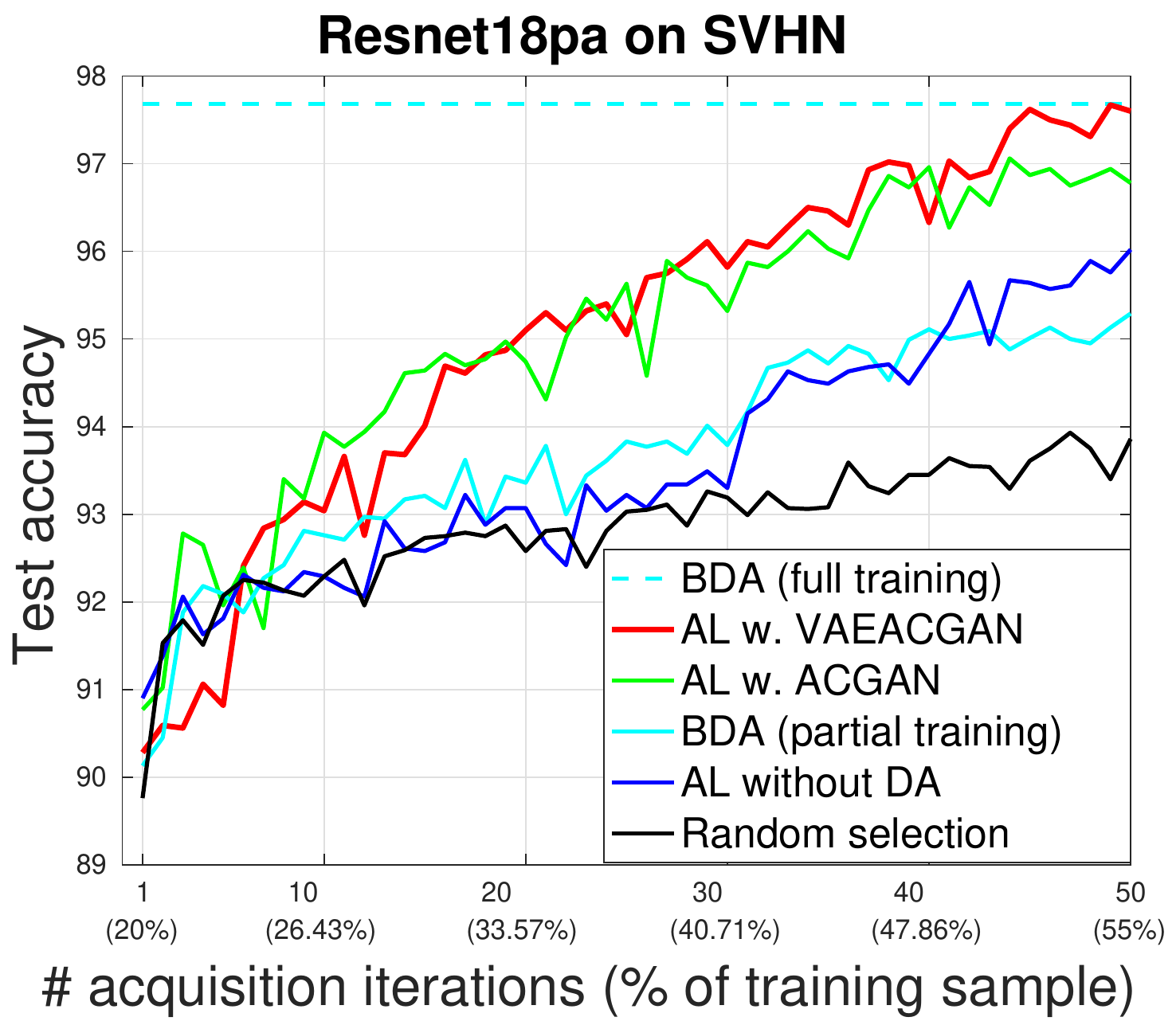}}
	\caption{Training and classification performance of the proposed Bayesian generative active learning (\emph{AL w. VAEACGAN}) compared to active learning using BDA~\cite{tran2017bayesian} (\emph{AL w. ACGAN}), BDA modeled with partial training sets (\emph{BDA (partial training)}), BALD~\cite{gal2017deep,houlsby2011bayesian} without data augmentation (\emph{AL without DA}), and random selection of training samples using the percentage of samples from the original training set  (\emph{Random selection}).  The result for BDA modeled with the full training set (\emph{BDA (full training)}) and $10 \times$ data augmentation represents an upper bound for all other methods.
		This performance is measured as a function of the number of acquisition iterations and respective percentage of samples from the original training set used for modeling.  First row shows these results using ResNet18~\cite{ResNet}, and second row shows ResNet18pa~\cite{he2016identity} on MNIST~\cite{lecun1998gradient} (column 1), CIFAR-10 (column 2) CIFAR-100~\cite{krizhevsky2012imagenet} (column 3), and SVHN~\cite{netzer2011reading} (column 4). }
	\label{fig:merged}
	\vskip -0.2in
\end{figure*}


Fig.~\ref{fig:merged} compares the classification performance of several methods as a function of the number of acquisition iterations and respective percentage of samples from the original training set used for modeling.  The methods compared are: 1) BDA~\cite{tran2017bayesian} modeled with the full training set (\emph{BDA (full training)}) and $10 \times$ data augmentation to be used as an upper bound for all other methods; 2) the proposed Bayesian generative active learning (\emph{AL w. VAEACGAN}); 3) active learning using BDA (\emph{AL w. ACGAN}); 4) BDA modeled with partial training sets (\emph{BDA (partial training)}); 5) BALD~\cite{gal2017deep,houlsby2011bayesian} without data augmentation (\emph{AL without DA}); and 6) random selection of training samples using the percentage of samples from the original training set (\emph{Random selection}).  Each point of the curves in Fig.~\ref{fig:merged} presents the result of one acquisition iteration, where each new point represents a growing percentage of the training set, as shown in the horizontal axis.  
In Fig.~\ref{fig:merged}, \emph{BDA (partial training)} relies on $2\times$ data augmentation, so it uses the same number of real and artificial training samples as \emph{AL w. VAEACGAN} and \emph{AL w. ACGAN} -- this enables a fair comparison between these methods.

\begin{table*}[th]
	\caption{Mean $\pm$ standard deviation of the classification accuracy on MNIST, CIFAR-10, and CIFAR-100 after 150 iterations over 3 runs}
	\label{tab:average_results}
	\vskip 0.15in
	\begin{center}
		\scalebox{0.76}{
			\begin{sc}
				\begin{tabular}{r c c c c c c}
					\toprule
					&\multicolumn{5}{c}{{MNIST}} \\
					& AL w. VAEACGAN & AL w. ACGAN & AL w. PMDA & AL without DA & BDA (partial training) & Random selection\\
					\midrule
					Resnet18 & $\mathbf{99.53 \pm 0.05} $  & $99.45 \pm 0.02$ & $99.37 \pm 0.15$ & $99.33 \pm 0.10$ & $99.33 \pm 0.04$ & $99.00 \pm 0.13$\\
					Resnet18pa & $\mathbf{99.68 \pm 0.08}$ & $99.57 \pm 0.07$ & $99.49 \pm 0.09$ & $99.35 \pm 0.11$ & $99.35 \pm 0.07$ & $99.20 \pm 0.12$\\
					\midrule
					&\multicolumn{5}{c}{{CIFAR-10}} \\
					Resnet18 & $\mathbf{87.63 \pm 0.11} $  & $86.80 \pm 0.45$ & $82.17 \pm 0.35$ & $79.72 \pm 0.19$ & $85.08 \pm 0.31$ & $77.29 \pm 0.23$\\
					Resnet18pa & $\mathbf{91.13 \pm 0.10}$ & $90.70 \pm 0.24$ & $87.70 \pm 0.39$ & $85.51 \pm 0.21$ & $86.90 \pm 0.27$ & $80.69 \pm 0.19$\\
					\midrule
					&\multicolumn{5}{c}{{CIFAR-100}} \\
					Resnet18 & $\mathbf{68.05 \pm 0.17} $  & $66.50 \pm 0.63$ & $55.24 \pm 0.57$ & $50.57 \pm 0.20$ & $65.76 \pm 0.40$ & $49.67 \pm 0.52$\\
					Resnet18pa & $\mathbf{69.69 \pm 0.13}$ & $67.79 \pm 0.76 $ & $59.67 \pm 0.60$ & $55.82 \pm 0.31$ & $65.79 \pm 0.51$ & $54.77 \pm 0.29$ \\
					
					\bottomrule
				\end{tabular}
			\end{sc}
		}
	\end{center}
	\vskip -0.1in
\end{table*}

To show a more informative comparison of our proposed approach (\emph{AL w. VAEACGAN}) with other methods presented in Fig.~\ref{fig:merged}, especially with \emph{AL w. ACGAN} and \emph{BDA (partial training)}, and active learning using PMDA (\emph{AL w. PMDA}), using Resnet18 and Resnet18pa on MNIST, CIFAR-10, and CIFAR-100, we ran the experiments three times (with different random initialisations) and show the final classification results (mean $\pm$ stdev) in Tab.~\ref{tab:average_results} (after 150 iterations).


\begin{figure}[th]
	\vskip 0.1in
	\centering
	\includegraphics[width=0.35\textwidth]{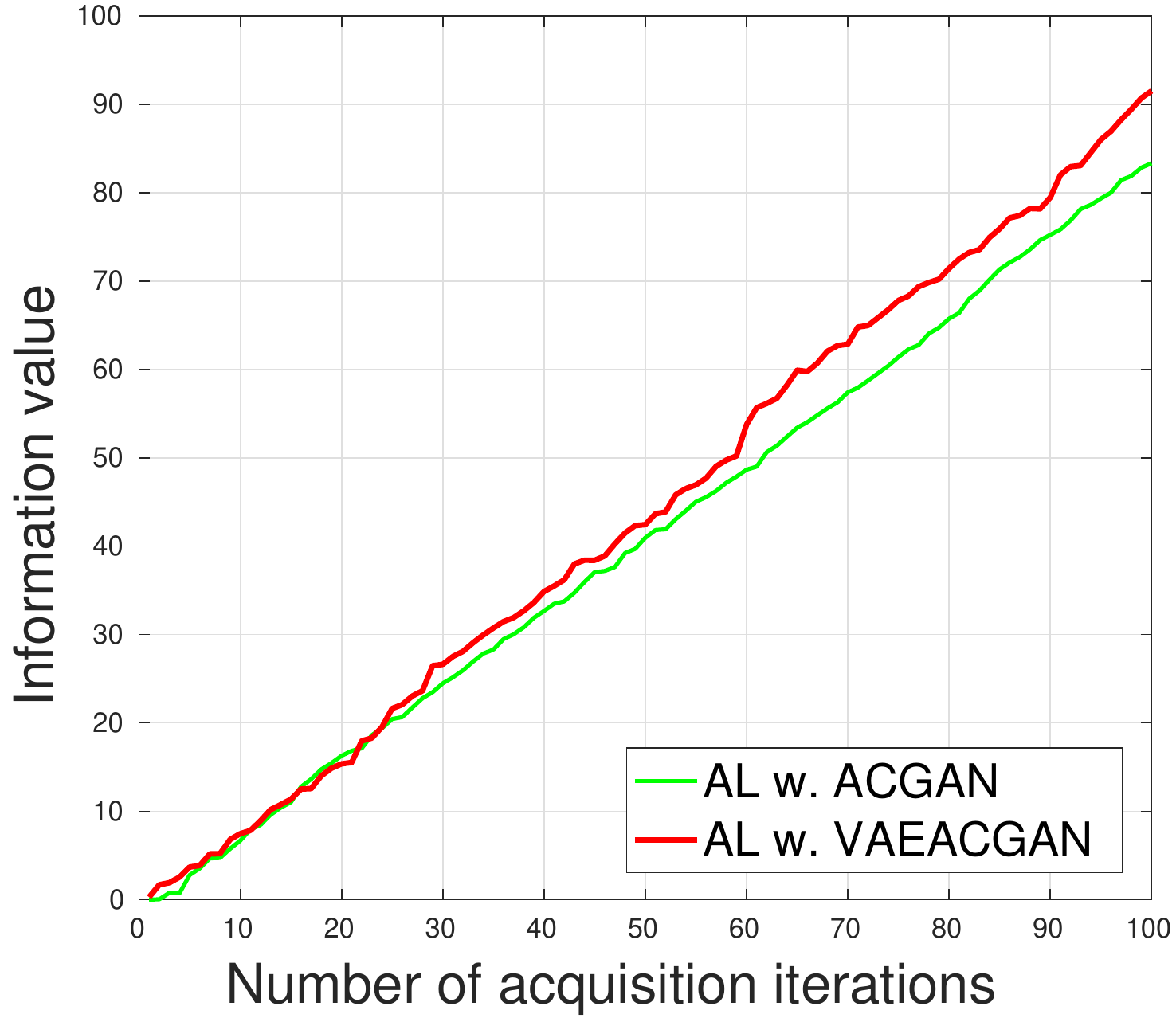}
	\caption{Average information value of samples measured by the acquisition function \eqref{acq_approx} of the samples generated by \emph{AL w. ACGAN} and \emph{AL w. VAEACGAN} using Resnet18 on CIFAR-100.}
	\label{fig:average_information}
	\vskip -0.1in
\end{figure}

We also compare the average information value of samples measured by the acquisition function \eqref{acq_approx} of the samples generated by \emph{AL w. ACGAN} and \emph{AL w. VAEACGAN} in Fig.~\ref{fig:average_information} using Resnet18 on CIFAR-100.

\begin{figure*}[th]
	\vskip 0.2in
	\centering
	\subfigure[MNIST]{\includegraphics[width=0.23\textwidth]{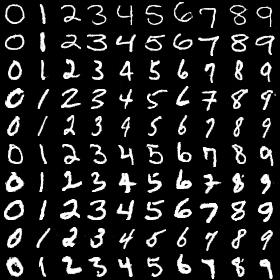}}
	\subfigure[CIFAR-10]{\includegraphics[width=0.23\textwidth]{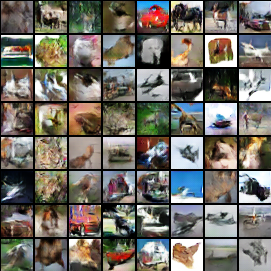}}
	\subfigure[CIFAR-100]{\includegraphics[width=0.23\textwidth]{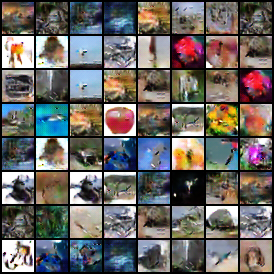}}
	\subfigure[SVHN]{\includegraphics[width=0.263\textwidth]{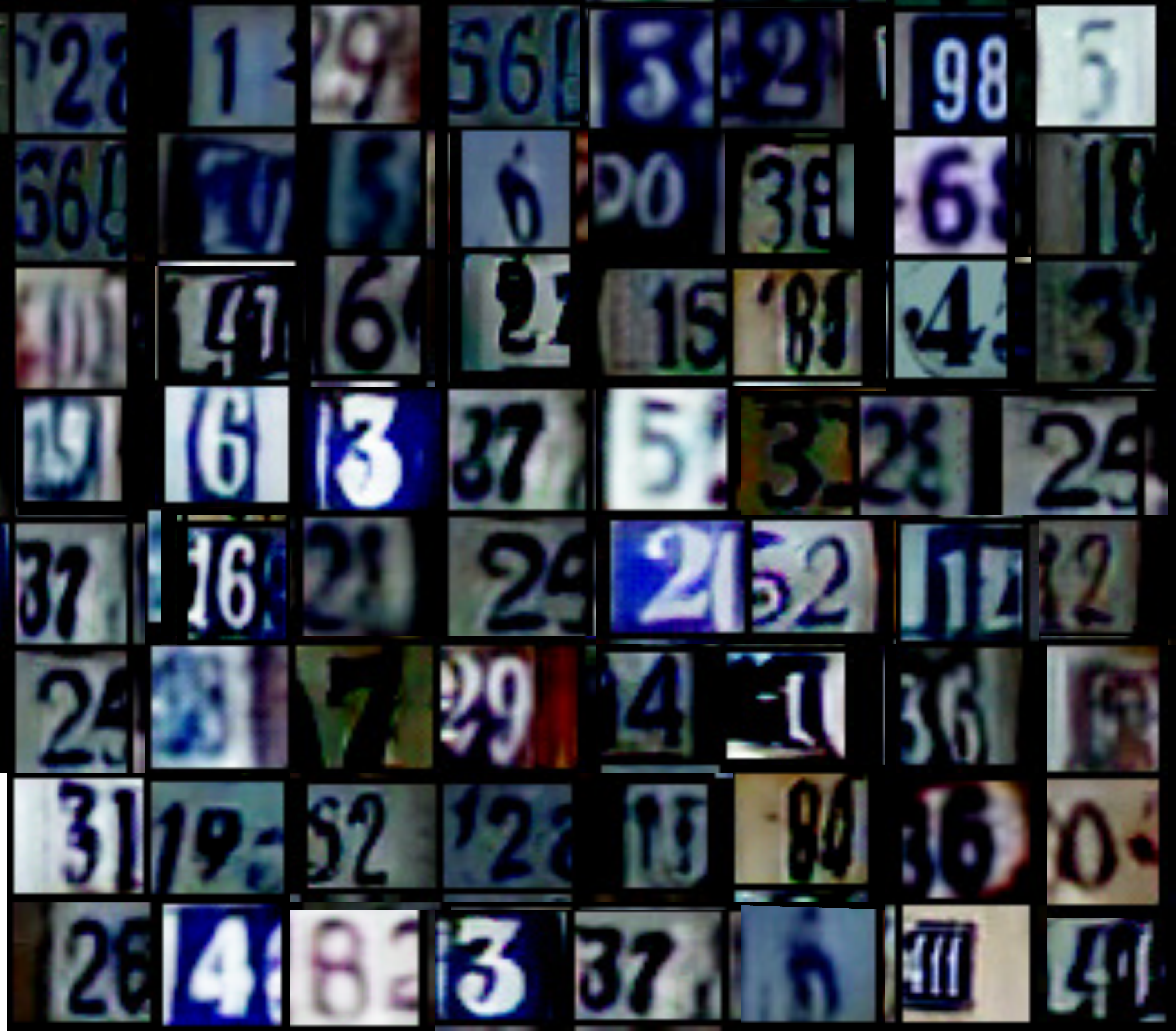}}
	\caption{Images generated by our proposed \emph{AL w. VAEACGAN} approach for each data set.}
	\label{fig:generation}
	\vskip -0.2in
\end{figure*}


Figure~\ref{fig:generation} displays images generated by our generative model for each data set.

\section{Discussion and Conclusions}\label{sec:conclusion}

Results in Fig.~\ref{fig:merged} consistently show (across different data sets and classification models) that our proposed Bayesian generative active learning (\emph{AL w. VAEACGAN}) is superior to active learning with BDA (\emph{AL w. ACGAN}), which is in fact an original model proposed by this paper.  Even though informative samples are used for training \emph{AL w. ACGAN}, the generated samples may not be informative, as depicted by Fig.~\ref{fig:average_information} which shows that samples generated by \emph{AL w. VAEACGAN} are more informative, particularly at latter stages of training.  Nevertheless, the samples generated by \emph{AL w. ACGAN} seem to be important for training given its better classification performance compared to \emph{AL without DA}.  
Table~\ref{tab:average_results} consistently shows that our proposed approach outperforms other methods on three data sets. In particular, the classification results by \emph{AL w. VAEACGAN} are statistically significant with respect to \emph{BDA (partial training)} on all those data sets, and with respect to \emph{AL w. ACGAN} on CIFAR-$\{10,100\}$ for both models (i.e., $p \leq .05$, two-sample t-test for ResNet18 and ResNet18pa).
Fig.~\ref{fig:merged} also shows that with a fraction of the training set, we are able to achieve a classification performance that is comparable with BDA using $10\times$ data augmentation over the entire training set -- this is evidence that the generation of informative training samples can use less human and computer resources for labeling the data set and training the model, respectively. When using MNIST and ResNet18, we let \emph{AL w. VAEACGAN} run until it reaches a competitive accuracy with BDA, which happened at 150 iterations -- this is then used as a stopping criterion for all methods. If we leave all models running for longer, both \emph{AL w. ACGAN} and \emph{AL w. VAEACGAN}  converge to \emph{BDA (full training)}, with \emph{AL w. VAEACGAN} converging faster.  Furthermore, results in Fig.~\ref{fig:merged} demonstrate that for training sets of similar sizes, our proposed \emph{AL w. VAEACGAN} produces better classification results than \emph{BDA (partial training)} for all experiments, re-enforcing the effectiveness of generating informative training samples. It can also be seen from Fig.~\ref{fig:merged} that, on MNIST, the active learning methods initially behave worse than random sampling, but after a certain number of training acquisition steps (around 75 steps and 13\% of the training set), they start to produce better results. Although the main goal of this work is the proposal of a better training process, the quality of the images generated, shown in Fig.~\ref{fig:generation}, is surprisingly high.

In this work we proposed a Bayesian generative active deep learning approach that consistently shows to be more effective than data augmentation and active learning in several classification problems. One possible weakness of our paper is the lack of a comparison with the only other method in the literature that proposes a similar approach~\cite{zhu2017generative}.  Although relevant to our approach, \cite{zhu2017generative} focuses on binary classification (that paper provides a brief discussion on the extension to multi-class, but does not show that extension explicitly), and the results shown in that paper are not competitive enough to be reported here. Note that our proposed approach is model-agnostic, it therefore can be combined with several currently introduced active learning methods such as~\cite{sener2018active,ducoffe2018adversarial,gissin2018discriminative}. 
In the future, we plan to investigate how to generate samples directly using complex acquisition functions, such as the one in \eqref{acq_approx}, instead of conditioning the sample generation on highly informative samples selected from the unlabeled data set. We also plan to work on the efficiency of our proposed method because its empirical computational cost is slightly higher than BDA~\cite{tran2017bayesian} and BALD~\cite{gal2017deep,houlsby2011bayesian}.

\section*{Acknowledgments}
We gratefully acknowledge the support by Vietnam International Education Development (VIED), Australian Research Council through grants DP180103232, CE140100016 and FL130100102.

\bibliography{mybib}
\bibliographystyle{icml2019}

\end{document}